\documentclass{article}

\usepackage{microtype}
\usepackage{graphicx}
\usepackage{subcaption}
\usepackage{booktabs} % for professional tables

\usepackage{hyperref}

\usepackage[preprint]{icml2026}

\usepackage{amsmath}
\usepackage{amssymb}
\usepackage{mathtools}
\usepackage{amsthm}

\usepackage[capitalize,noabbrev]{cleveref}

\theoremstyle{plain}
\newtheorem{theorem}{Theorem}[section]

\newtheorem{lemma}{Lemma}[section]

\theoremstyle{definition}
\newtheorem{definition}{Definition}[section]
\usepackage{amsthm}

\newtheoremstyle{compactstyle}
  {3pt}                % 上方间距
  {3pt}                % 下方间距
  {\itshape}           % 正文字体（通常假设用斜体）
  {}                   % 缩进量
  {\bfseries}          % 标题头字体
  {.}                  % 标题后标点
  {.5em}               % 标题后间距
  {}                   % 标题头说明

\theoremstyle{compactstyle}

\newtheorem{assumption}{Assumption}[section]
\theoremstyle{remark}

\usepackage[textsize=tiny]{todonotes}

\begin{document}

\twocolumn[
  \icmltitle{RepFlow: Representation Enhanced Flow Matching for Causal
Effect Estimation}



  \icmlsetsymbol{equal}{*}
  \begin{icmlauthorlist}
  	\icmlauthor{Yifei Xie}{yyy}
  	\icmlauthor{Jian Huang}{yyy}
  
  \end{icmlauthorlist}
  
  \icmlaffiliation{yyy}{Department of Data Science and Artificial Intelligence, The
  	Hong Kong Polytechnic University, Hong Kong, China}

  \icmlcorrespondingauthor{Jian Huang}{j.huang@polyu.edu.hk}

  \icmlkeywords{Machine Learning, ICML}

  \vskip 0.3in
]



\printAffiliationsAndNotice{}  % no special notice (required even if empty)

\begin{abstract}
 Estimating causal effects from observational data has become increasingly critical in diverse fields including healthcare, economics, and social policy. The fundamental challenge in  causal inference arises from the missing counterfactuals and the selection bias. Existing methods are largely limited to point estimates and lack the capacity for distribution modeling. In this work, we propose RepFlow, a novel framework that formulates causal effect estimation as a joint optimization problem integrating representation learning with Conditional Flow Matching (CFM).
 RepFlow mitigates selection bias by minimizing the entropically regularized Wasserstein distance between treated and control representations.
 To enhance numerical stability, we further introduce an $L_2$ normalization constraint on latent representations.
 This balanced representation enables the flow model to accurately capture the distribution of potential outcomes. Extensive experiments across a wide range of benchmarks demonstrate that RepFlow consistently outperforms existing methods in both point and distributional causal effect estimation.
\end{abstract}

\section{Introduction}

Estimating causal effects plays an essential role across diverse fields. While traditional predictive tasks typically identify statistical associations, causal inference seeks to answer counterfactual 'what-if' questions, providing critical guidance for effective decision-making \cite{pearl2009causality}.
For instance, in healthcare, clinicians need to determine the optimal treatment plan for a patient by assessing the predicted outcomes for different patient groups \cite{corrao2012external}. In social policy, governments endeavor to determine whether subsidized job-training programs actually improve workers' employment prospects \cite{britto2022effect}. Similarly, in digital marketing, platforms seek to predict whether an advertisement will successfully persuade a user to purchase a product \cite{rzepakowski2012decision}. 

The fundamental challenge of causal inference is that we can observe outcomes only for the treatment received, not for alternatives. Randomized controlled trials (RCTs) are widely regarded as the gold standard for identifying causal effects, as random treatment assignment ensures comparability between treated and control groups \cite{hernan2010causal}. In practice, however, RCTs  are often infeasible owing to ethical constraints and high costs. Consequently, the growing availability of large-scale observational data has further increased interest in causal inference from observational studies \cite{rosenbaum1983central}. In such settings, treatment assignment is typically non-random as confounding variables influence it. This leads to selection bias, also known as covariate shift, in which the distributions of measured features differ across treatment groups \cite{kirklin2010second}. For instance, in clinical practice, a newly approved medication is more often prescribed to younger patients, while older patients are more likely to remain on standard treatment. In this case, age becomes a confounder that renders the two groups incomparable. Unless these distribution shifts are adequately addressed, models will fail to provide accurate causal estimates.

To address selection bias in observational data, traditional causal inference research has developed a variety of techniques primarily focused on estimating the Average Treatment Effect (ATE). Approaches such as propensity score matching and weighting aim to balance covariate distributions \cite{lunceford2004stratification,imai2014covariate}. Doubly robust estimators combine outcome modeling with propensity scores for improving robustness against model misspecification \cite{chernozhukov_doubledebiased_2018}. However, these methods generally provide population-level estimates and overlook treatment effect heterogeneity.  Consequently, there has been a growing interest in estimating the Conditional Average Treatment Effect (CATE) to capture subgroup-specific effects. This includes various meta-learners, such as the S-learner, T-learner, and DR-learner, which split the CATE estimation into supervised learning sub-problems \cite{athey2015machine,kunzel2019metalearners,kennedy2023towards}, as well as tree-based approaches such as Causal Forests \cite{wager2018estimation}. Furthermore, deep representation learning methods have attracted considerable attention by learning balanced embeddings that minimize the distance between treatment and control distributions. These approaches consider a range of frameworks, such as using integral probability metrics (IPM) for regularization or harnessing optimal transport \cite{shalit2017estimating,hassanpour2019learning,cheng2022learning,wang2023optimal}. A notable shortcoming of the aforementioned methods is that they primarily focus on point estimates, which are unable to capture the full distribution of predicted outcomes. Learning the entire distribution is necessary for reliable decision-making, as it provides fundamental information about the uncertainty of treatment effects \cite{gische2022beyond,kennedy2023semiparametric}. 

To mitigate this constraint,  researchers have turned to generative modeling as a compelling technique for learning the distribution of predicted outcomes. Early work in this area leveraged Generative Adversarial Networks (GANs), such as GANITE, to synthesize potential outcomes (POs) \cite{yoon2018ganite}. However, GAN-based frameworks commonly suffer from training instability and mode collapse. In recent years, diffusion-based methods have become a popular choice for causal inference \cite{chao2023interventional,ma2024diffpo,chen2025enhancing}, offering a flexible way to model complex data distributions. However, the SDE-based sampling process in these models is inefficient, requiring many iterations and resulting in slow sampling speeds. The random characteristic of the process can also lead to numerical instability in the generated POs. To handle these issues, PO-Flow adopts a Flow Matching framework based on Ordinary Differential Equations (ODEs) \cite{wu2025po}. This ODE-based solution provides more efficient inference and generates more stable, deterministic sample paths. However, despite these improvements to the generative mechanism, PO-Flow does not address the fundamental issue of selection bias inherent to observational studies.

In this work, we present a representation enhanced flow matching framework, called RepFlow, that integrates balanced representation learning with generative causal modeling. At the core of our approach is a joint objective function that simultaneously optimizes a representation encoder and a conditional flow matching model. Specifically, we mitigate selection bias by minimizing the Wasserstein discrepancy between treatment groups in a latent space, which is computed efficiently via entropic regularization. To ensure computational stability and prevent gradient explosions during optimal transport optimization, we enforce L2 normalization in the representation layer to smoothly map latent features onto a unit hypersphere. This balanced foundation enables the conditional flow matching model to learn the full density of POs accurately. By using an ODE formulation to map noise to output distributions, RepFlow provides deterministic, interpretable sample paths that enable fast, accurate sampling during inference. This unified architecture ensures the generative process remains robust even under substantial covariate shift. Ultimately, our framework delivers a unified, flexible approach that enables fast, accurate sampling and reliable causal estimation across a wide range of complex settings.

Our main contributions are summarized as follows:
\begin{itemize}
	\item We propose RepFlow, a unified framework that integrates balanced representation learning with conditional flow matching to mitigate selection bias and model the full distribution of potential outcomes (POs).
	
	\item We design a stable joint objective by combining entropically regularized Wasserstein distance with $L_2$-normalized representations, effectively improving numerical stability and preventing gradient explosion during training.
	
	\item We provide a theoretical analysis establishing a distributional generalization bound, showing that the expected distributional error is controlled by the factual distribution losses and the Wasserstein distance between treatment groups in the latent space, which motivates our joint optimization strategy.

	\item Extensive experiments on IHDP, ACIC 2018, and synthetic datasets demonstrate that RepFlow consistently outperforms state-of-the-art baselines in both point and distributional causal effect estimation.
\end{itemize}
\section{Related Work}
\subsection{CATE Estimation}
Existing literature on CATE estimation generally falls into two categories: model-agnostic and model-specific approaches. Model-agnostic estimators decompose the causal task into standard supervised learning problems. These include one-step "plug-in" learners, such as the S-learner and T-learner, which estimate POs directly to compute their difference \cite{athey2015machine,kunzel2019metalearners}. However, these methods often fail to adequately address selection bias. In contrast, two-step learners, such as the X-learner and DR-learner, first utilize nuisance functions to construct pseudo-outcomes and subsequently regress these values onto input covariates to obtain the CATE \cite{he2022x,kennedy2023towards}.  These strategies commonly assume that the CATE function has a simpler functional structure than the underlying POs. In finite-sample settings, this assumption can introduce a specific inductive bias that is often too restrictive, potentially resulting in suboptimal performance in predicting individualized outcomes.

Model-specific estimators adapt specialized machine learning techniques directly for causal inference tasks. This category includes tree-based approaches, exemplified by Causal Forest, which modify random forest objectives to estimate heterogeneous treatment effects during tree building \cite{wager2018estimation}.  Neural network architectures like CFRNet and DragonNet have gained popularity \cite{shalit2017estimating,shi2019adapting}; these methods rely on representation learning to align covariate distributions between treatment and control groups, therefore alleviating selection bias \cite{shalit2017estimating,hassanpour2019learning,du2021adversarial,cheng2022learning,wang2023optimal}. Other approaches, such as CEVAE, incorporate latent-variable models to address hidden confounding while modeling outcome distributions \cite{louizos2017causal}. However, most of these approaches primarily target CATE estimation rather than predicting individualized POs. Since these models are typically restricted to point estimates, they lack the capacity to provide full distributional information regarding POs.

\subsection{Generative models for causal inference}
Generative models were introduced for learning complex data distributions from a given dataset, enabling the generation of high quality samples that reflect underlying data variance \cite{goodfellow2020generative,ho2020denoising,zhou2023deep,chao2023interventional,chen2025enhancing}. Building on this ability, recent studies have applied generative methods to causal inference related tasks \cite{SanchezLOT23}, including generating counterfactual images \cite{deng2024causal}, performing causal discovery, and answering causal queries \cite{Sánchez-Martin_Rateike_Valera_2022}. While these methods indicate the flexibility of generative modeling for causal reasoning, they are primarily designed for structural equation or image counterfactual generation, rather than for directly estimating causal effects from observational data.

Another line of work applies generative models to causal treatment effect estimation under the POs framework. GANITE introduces a two-stage adversarial architecture in which a counterfactual generator produces proxy outcomes to supervise an ITE generator \cite{yoon2018ganite}; however, this approach is prone to resulting in instability during adversarial training. Diffusion-based approaches have also been proposed for modeling POs \cite{ma2024diffpo}. Although DiffPO incorporates inverse propensity weighting (IPW) into the diffusion noise loss to address treatment imbalance, this design introduces substantial numerical bias and variance, leading to unstable training when propensity scores are near 0 or 1. In addition, the method incurs high computational cost due to the large number of sampling iterations required for POs generation. More recently, PO-Flow adopts flow matching to generate POs with deterministic dynamics and improved sampling efficiency. Despite its efficiency, PO-Flow does not explicitly mitigate selection bias, making it sensitive to distribution shifts between treated and control groups. To bridge these gaps, we propose RepFlow, which combines balanced representation learning with a flow matching framework, enabling unbiased causal effect estimation without reliance on IPW weighting.

\section{Preliminaries}
In this section, we establish the formal notation for causal effect estimation and introduce the fundamentals of Continuous Normalizing Flows and the Flow Matching training paradigm.
\begin{figure*}[ht]  % 跨双栏，优先页顶（ICML模板强制要求）
	\centering
	% 1. 增大子图宽度占比（从0.45→0.48，接近跨栏总宽度的一半）
	% 2. 降低子图间间隙（从0.05→0.02，减少留白）
	% 子图1：Training Phase（左）
	\begin{subfigure}{0.48\textwidth}  
		\centering
		% 关键：优先用width控制（适配\textwidth），height作为兜底且适当增大
		% keepaspectratio保证图片不变形
		\includegraphics[width=\linewidth, height=20cm, keepaspectratio]{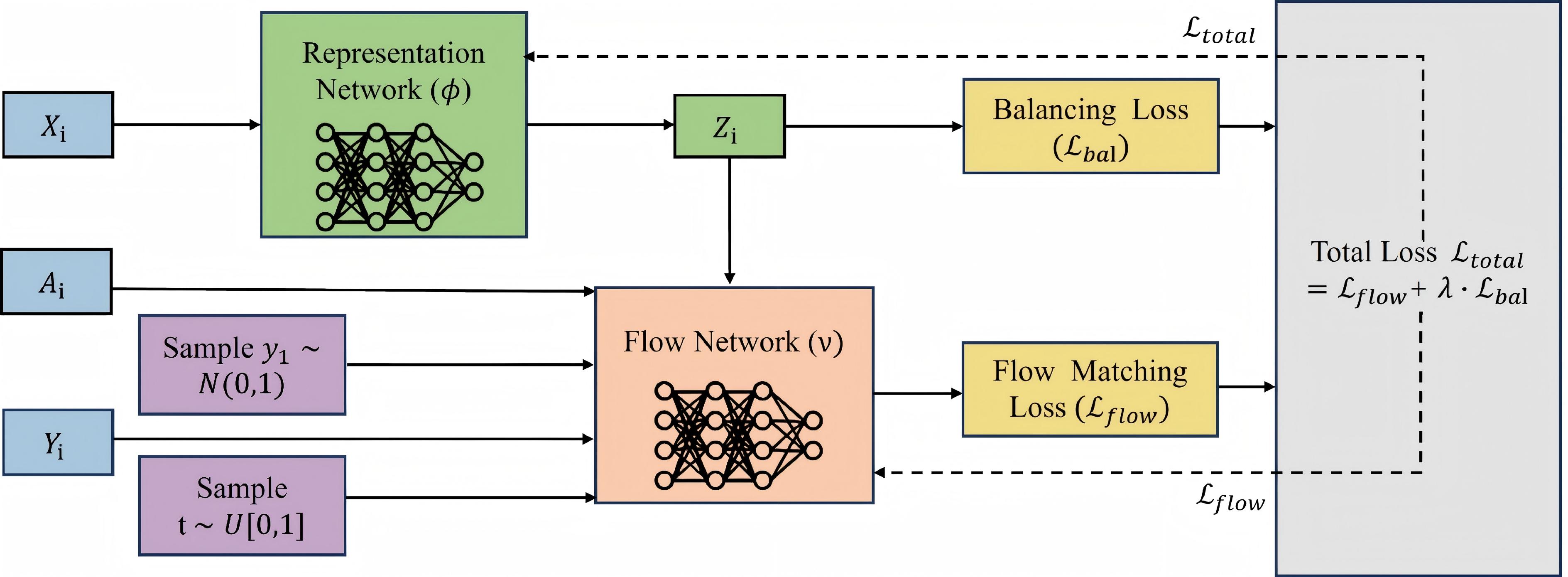}
		\caption{Training Phase}
		\label{Fig.sub.3}
	\end{subfigure}
	\hspace{0.02\textwidth}  % 缩小子图间距，释放宽度空间
	% 子图2：Inference Phase（右）
	\begin{subfigure}{0.48\textwidth}
		\centering
		% 保持与左图参数一致，确保对称
		\includegraphics[width=\linewidth, height=20cm, keepaspectratio]{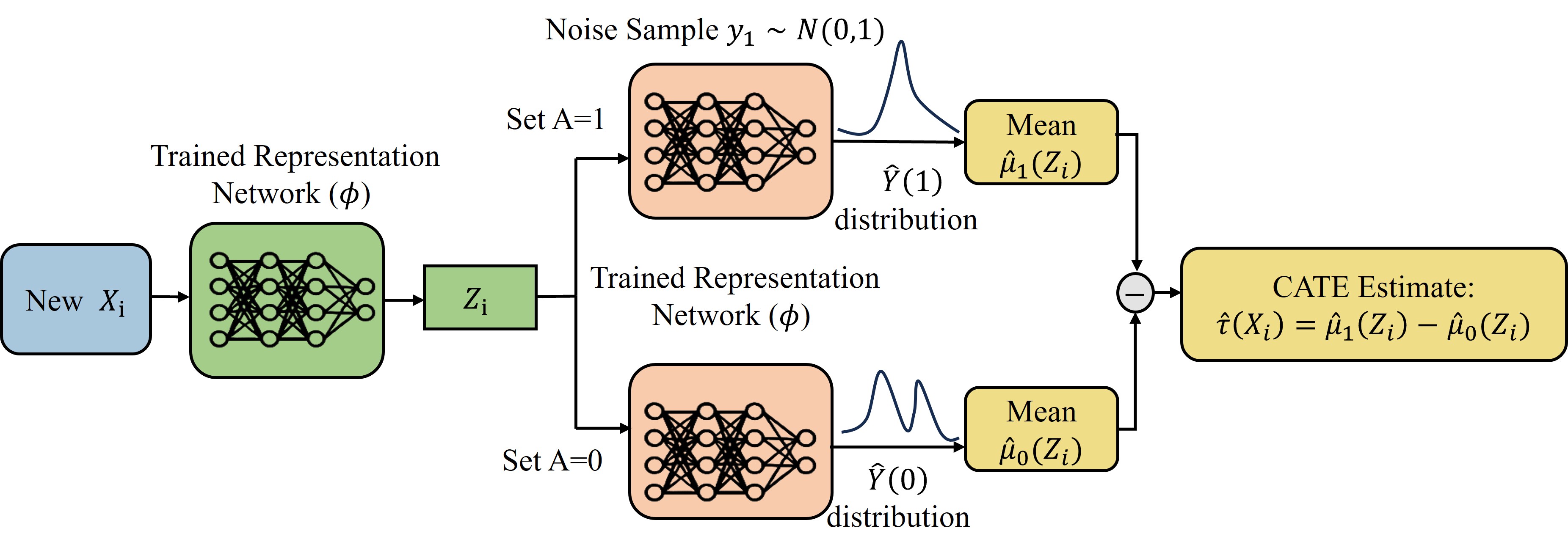}
		\caption{Inference Phase}
		\label{Fig.sub.4}
	\end{subfigure}
	
	\caption{Architecture of Repflow.}
	\label{fig1}
\end{figure*}
\subsection{Causal Inference and Potential Outcomes}
We consider an observational dataset $\mathcal{D} = \{(x_i, a_i, y_i)\}_{i=1}^n$ drawn from a distribution $p(X, A, Y)$, where $x \in \mathbb{R}^{d_X}$ denotes the observed covariates, $a \in \{0, 1\}$ is the binary treatment assignment, and $y \in \mathbb{R}^{d_Y}$ is the observed outcome. 

Following the potential outcomes framework \cite{rubin1974estimating}, let $Y^{(0)}$ and $Y^{(1)}$ represent the potential outcomes under control and treatment, respectively. The observed outcome is given by $Y = AY^{(1)} + (1-A)Y^{(0)}$. Our objective is to estimate the CATE, defined as $\tau(x) = \mathbb{E}[Y^{(1)} - Y^{(0)} \mid X=x]$. In order to identify  causal effects from observational data, we need to make the following standard assumption:

\begin{assumption}
\label{assump:core_causal}
We assume the standard causal identification conditions:
(i) \emph{Consistency}, i.e., if an individual receives treatment $A=a$, then
$Y = Y^{(a)}$;
(ii) \emph{Unconfoundedness}, i.e.,
$\{Y^{(0)}, Y^{(1)}\} \perp\!\!\!\perp A \mid X$;
and (iii) \emph{Overlap}, i.e.,
$0 < \mathbb{P}(A=1 \mid X=x) < 1$ for all $x \in \mathcal{X}$.
\end{assumption}

Under Assumption \ref{assump:core_causal}, the distribution of potential outcomes can be identifd from observational data as $p(Y^{(a)} \mid X) = p(Y \mid X, A=a)$. 
Consequently, the CATE denoted as $\tau$, can be identified through the following expansion of expectations:
\begin{align}
\tau &= \mathbb{E}[Y(1) \mid X] - \mathbb{E}[Y(0) \mid X] \nonumber \\
     &= \mathbb{E}[Y \mid A=1, X] - \mathbb{E}[Y \mid A=0, X]. \nonumber
\end{align}

\begin{definition}
	Let $\Phi: \mathcal{X} \rightarrow \mathcal{Z}$ be the representation network, and let $\hat{p}_{\theta}(Y|\phi(x), a)$ be the conditional distribution of potential outcomes generated by the Repflow model. For a unit $x \in \mathcal{X}$ and treatment assignment $a \in \{0, 1\}$, the  distributional distance is defined as:
	\begin{equation}
		d_{\theta, \phi}(x, a) = \mathcal{D} \left( p(Y|x,a), \hat{p}_{\theta}(Y|\phi(x), a) \right), \nonumber
	\end{equation}
	where $\mathcal{D}(\cdot, \cdot)$ is a metric on the space of probability distributions, such as the Wasserstein distance $\mathcal{W}_1$.
\end{definition}

\begin{definition} \label{zdf3}
	Conditioned on the treatment assignment, the factual distribution losses for each group are:
	\begin{align}
		& \epsilon_{F}^{a=1} = \int_{\mathcal{X}} d_{\theta, \phi}(x, 1) p(x|a=1) dx,  \nonumber \\
		& \epsilon_{F}^{a=0} = \int_{\mathcal{X}} d_{\theta, \phi}(x, 0) p(x|a=0) dx. \nonumber
	\end{align}
\end{definition}
\begin{definition}
	The expected distributional estimation error is defined as the total discrepancy between the true and estimated potential outcome distributions over the population:
	\begin{equation}
		\epsilon_{D}(\theta, \phi) = \int_{\mathcal{X}} \left( d_{\theta, \phi}(x, 1) + d_{\theta, \phi}(x, 0) \right) p(x) dx.
	\end{equation}
\end{definition}

While many existing works focus primarily on estimating the response function $\mu_a(x) = \mathbb{E}[Y^{(a)} \mid X=x]$ to compute the point estimate of $\tau$, our approach aims to model the full conditional distribution $p(Y \mid X, A=a)$.

\subsection{Continuous Normalizing Flows and Flow Matching}
\noindent\textbf{Continuous Normalizing Flows (CNFs).} CNFs are a class of generative models that transform a simple prior $p_1$ (e.g., standard Gaussian) into a complex data distribution $p_0$ via an invertible ODE-based mapping \cite{chen2018neural,gao2024convergence}. Unlike discrete normalizing flows \cite{rezende2015variational}, which stack finite invertible layers, CNFs model distributions through continuous trajectories, offering greater flexibility and avoiding the gradient instability and costly Jacobian computations associated with deep discrete flows. Formally, let $y_t \in \mathbb{R}^{d_Y}$ for $t\in[0,1]$ follow a trajectory governed by a learnable velocity field $v_\theta(\psi_t, t)$:
\begin{equation}
	\frac{dy_t}{dt} = v_\theta(\psi_t, t), \quad y_0 \sim p_0, \quad y_1 \sim p_1,
	\label{eq:ode_main}
\end{equation}
where $v_\theta$ is parameterized by a neural network.

\noindent\textbf{Flow Matching (FM).} FM trains a CNF by regressing the velocity field $v_\theta$ onto a target vector field $u_t(x)$ that induces a desired probability path $p_t(x)$:
\begin{equation}
	\mathcal{L}_{\text{FM}}(\theta) = \mathbb{E}_{t \sim \mathcal{U}[0,1], x \sim p_t(x)} \| v_\theta(x, t) - u_t(x) \|^2.
\end{equation}
However, naively optimizing $\mathcal{L}_{\text{FM}}$ is intractable in practice because the marginal vector field $u_t(x)$ and the path $p_t(x)$ generally lack a closed-form expression. 

To address this, we adopt Conditional Flow Matching (CFM), which decomposes the marginal vector field into sample-conditional vector fields $u_t(x|y_0)$ generating conditional paths $p_t(x|y_0)$. A key theoretical foundation of this approach is that the gradient of the CFM loss is equivalent to that of the original FM loss, $\nabla_\theta \mathcal{L}_{\text{FM}}(\theta) = \nabla_\theta \mathcal{L}_{\text{CFM}}(\theta)$, as formally established in Theorem 2 of \cite{LipmanCBNL23}.

We define a linear interpolant between an observed outcome $y_0 \sim p_0$ and noise $y_1 \sim N(0,I)$:
\begin{equation}
	\psi_t(y_0, y_1) = (1 - t)y_0 + (t + \sigma(1 - t))y_1, \nonumber
\end{equation}
\begin{equation}
	u_t(y_0, y_1) = \frac{d}{dt}\psi_t(y_0, y_1) = (1 - \sigma)y_1 - y_0, \nonumber
\end{equation}
where $\sigma > 0$ ensures $\psi_0 \sim N(y_0, \sigma^2 I)$. The CFM objective becomes
\begin{equation}
	\mathcal{L}_{\text{CFM}}(\theta) = \mathbb{E}_{t, y_0, y_1} \| v_\theta(\psi_t(y_0, y_1)) - u_t(y_0, y_1) \|^2.
\end{equation}

Optimizing this objective allows us to predict potential outcomes (POs). Unlike PO-Flow, we integrate balanced representations $\phi(X)$ to mitigate selection bias in observational data.

\section{Methodology} \label{sec:method}

In this section, We propose a unified framework that combines balanced representation learning with flow-based generative modeling for estimating POs. As shown in Figure \ref{fig1}, our framework comprises of two primary components: (1) a representation learning module that projects covariates into a balanced latent space to address selection bias, and (2) a conditional flow matching module that generates potential outcomes based on these balanced representations.

\begin{theorem}\label{theo}
	Let $\Phi: \mathcal{X} \rightarrow \mathcal{Z}$ be the representation network. Then, the expected distributional error over the population satisfies:
	\begin{equation}
		\epsilon_{D}(\theta, \phi) \leq \epsilon_{F}^{a=0} + \epsilon_{F}^{a=1} + B_{\Phi} \cdot \mathcal{W}\!\left(p_{\Phi}^{a=1}, p_{\Phi}^{a=0}\right),
	\end{equation}
	where $\mathcal{W}(\cdot,\cdot)$ denotes the Wasserstein distance between distributions and $B_{\Phi}$ is a constant.
\end{theorem}
\begin{proof}
	The detailed proof is provided in Appendix~\ref{proof}.
\end{proof}

Theorem \ref{theo} provides the theoretical foundation for our proposed joint loss function by establishing a upper bound on the expected distributional error. It formally demonstrates that simultaneously minimizing the  flow matching loss and the latent distribution discrepancy leads to a principled control of the overall estimation error. This result motivates our integrated optimization strategy and explains its robustness under selection bias.
\subsection{Balanced Representation Learning}
In observational studies, treatment assignments are typically non-random, which introduces selection bias. For example, in clinical settings, patients exhibiting more severe symptoms are more likely to receive aggressive treatments. This scenario represents a distributional shift in covariates, expressed as $p(X \mid A = 0) \neq p(X \mid A = 1)$. Neglecting this shift can result in substantial bias and increased variance in estimates of POs, particularly in regions with low overlap.

%Representation learning for causal inference aims to learn a mapping $\phi: \mathcal{X} \to \mathcal{Z}$ that transforms covariates $X$ into a balanced representation space $\mathcal{Z}$. In particular, we employ a multi-layer feed-forward neural network $\phi$ so that the resulting representation $Z = \phi(X)$ is invariant with respect to treatment assignment $A$. The empirical distributions of representations for the control and treatment groups are denoted as $\hat{P}_\phi^{a=0} = \{z_i^0\}_{i=1}^{n}$ and $\hat{P}_\phi^{a=1} = \{z_j^1\}_{j=1}^{m}$, respectively.

To mitigate selection bias, we minimize the discrepancy between these distributions using the Wasserstein distance as the metric.  The empirical distributions of representations for the control and treatment groups are denoted as $\hat{P}_\phi^{a=0} = \{z_i^0\}_{i=1}^{n}$ and $\hat{P}_\phi^{a=1} = \{z_j^1\}_{j=1}^{m}$, respectively.
While previous approaches often used the maximum mean discrepancy (MMD) or relied on dual-form Wasserstein approximations that impose complex Lipschitz constraints \cite{shalit2017estimating,hassanpour2019learning,kong2023covariate,liu2025wasserstein}, we adopt the optimal transport (OT) framework \cite{villani2008optimal}. Specifically, the Wasserstein distance is defined as the minimum cost of transporting mass between distributions \cite{kantorovich2006translocation}:
\begin{equation}\label{bal_loss}
\mathcal{L}_{bal}=\mathcal{W}(\hat{P}_\phi^{a=0}, \hat{P}_\phi^{a=1}) = \min_{\pi \in \Pi(\hat{P}_\phi^{a=0}, \hat{P}_\phi^{a=1})} \langle H, \pi \rangle,
\end{equation}
where $H \in \mathbb{R}_{+}^{n \times m}$ is the cost matrix with $H_{ij} = \|z_i^0 - z_j^1\|_2$ representing the Euclidean distance between representations, and $\Pi$ denotes the set of all feasible transport plans. Since solving the Kantorovich problem is computationally
intensive, we introduce entropic regularization to enable
efficient training. This regularization allows us to compute
the regularized transport plan via the sinkhorn algorithm \cite{distances2013lightspeed}.

A key technical refinement in our framework is the enforcement of $L_2$ normalization at the output layer of the representation network $\phi$, ensuring $z = \phi(x) / \|\phi(x)\|_2$. This strategy constrains all learned representations to lie on a unit hypersphere, yielding several critical advantages. This improves numerical stability by bounding the cost matrix $H$, preventing extreme values in Sinkhorn iterations and gradient explosion. It also ensures fully differentiable objectives with stable gradients, enabling more reliable and efficient updates This improves numerical stability by bounding the cost matrix $H$, preventing extreme values in Sinkhorn iterations and gradient explosion. It also ensures fully differentiable objectives with stable gradients, enabling more reliable and efficient updates to $\phi$.

By minimizing $\mathcal{L}_{\text{bal}}$, we ensure that the learned representations $Z$ are balanced across treatment groups, thereby mitigating selection bias and providing a robust foundation for the subsequent Flow Matching process.

\subsection{Conditional Flow Matching for Potential Outcomes}To model the  conditional distribution of POs, we utilize the balanced representation $Z$ as the conditioning variable within the CFM framework. A major challenge in balanced representation learning is ensuring that the mapping $\phi$ does not discard essential information for predicting $Y$ while minimizing the distribution discrepancy. In our framework, we address this by jointly optimizing the representation network and the velocity field.

To this end, we define a velocity field $v_\theta(\psi_t, t, z, a)$ conditioned on the learned representation $Z$ and the treatment assignment $a$. The objective of the flow matching task is to minimize the following loss:
\begin{equation}
\label{flow_loss}
\mathcal{L}_{\text{flow}}(\theta) = \mathbb{E}_{t, y_0, y_1,x,a} || v_\theta(\psi_t, t, \phi(x), a) - u_t(y_0,y_1) ||^2,
\end{equation}
where $u_t$ is the conditional vector field transporting the base distribution to the outcome distribution.

To achieve a representation that is both balanced across treatment groups and highly informative for outcome estimation, we define the total objective function as a weighted sum of the balancing loss and the flow matching loss:
\begin{equation}\mathcal{L}_{\text{total}} = \mathcal{L}_{\text{flow}}(\theta) + \lambda \mathcal{L}_{\text{bal}}(\phi),\end{equation}where $\lambda > 0$ is a hyperparameter governing the trade-off between balance and predictive power. Under this joint optimization target , the representation network $\phi$ receives gradients from both $\mathcal{L}_{\text{bal}}$ and $\mathcal{L}_{\text{flow}}$. While $\mathcal{L}_{\text{bal}}$  encourages $Z$ to be treatment invariant, $\mathcal{L}_{\text{flow}}$ acts as a predictive constraint that prevents $Z$ from collapsing into an uninformative state. By extracting outcome relevant information, $\phi(x)$ establishes a robust conditioning manifold for the velocity field. The training of RepFlow involves the joint optimization of the representation encoder $\phi$ and the velocity network $v_\theta$ by minimizing $\mathcal{L}_{\text{total}}$.
The detailed training procedure is provided in Algorithm \ref{alg:repflow_train}.
% This enables the flow model to  precisely describe the conditional distribution of POs, supporting accurate estimation of individual treatment effects.

\noindent\textbf{POs Generation and Causal Effect Estimation.} In the RepFlow framework, the estimation of POs is formulated as a continuous trajectory integration problem. Once the conditioned velocity field $v_\theta$ is trained, it defines a deterministic probability path that maps standard Gaussian noise to the outcome distribution. This enables precise PO estimation by leveraging the learned dynamics.

Specifically, for a given individual $x$ and treatment assignment $a \in \{0, 1\}$, we first compute the representation $z = \phi(x)$. To generate a potential outcome $\hat{y}^{(a)}$, we initialize a latent noise $y_1 \sim \mathcal{N}(0, I)$ at $t=1$. The potential outcome is then obtained by solving the following reverse-time ODE:
\begin{equation}
\hat{y}^{(a)} = y_1 - \int_{0}^{1} v_\theta(\psi_t, t; z, a)  dt.
\label{eq:reverse_integral}
\end{equation}
In our implementation, this continuous-time integral is solved using the Runge-Kutta method \cite{butcher1996history} to ensure numerical stability and precision trajectory tracking from the latent space  to the outcome space.

Based on the identification results in Assumption \ref{assump:core_causal}, the causal effect is estimated by comparing the expected potential outcomes.
We draw $M$ independent samples $\{y_{1}^{(m)}\}_{m=1}^M$ from the prior distribution $\mathcal{N}(0, I)$ and generate the corresponding synthetic outcomes $\{y^{(a,m)}\}_{m=1}^M$ using the reverse-time process defined in Eq.~\eqref{eq:reverse_integral}. The response functions are then empirically estimated as
\begin{equation}
\hat{\mu}_a(z) = \frac{1}{M} \sum_{m=1}^{M} y^{(a,m)}, \quad a \in \{0,1\}.
\end{equation}
The individual treatment effect  for an individual with covariates $x$ is subsequently estimated as
\begin{equation}
\hat{\tau}(x) = \hat{\mu}_1(\phi(x)) - \hat{\mu}_0(\phi(x)).
\end{equation}

By leveraging the set of generated samples $\{y^{(a, m)}\}_{m=1}^M$, our approach enables the construction of an empirical estimation of the conditional distribution $p(Y \mid X=x, A=a)$. Such distributional knowledge is vital for accounting for uncertainty in practical applications; it not only informs how likely a specific outcome is but also provides the probability that a potential outcome lies within a desired range. By moving beyond point estimation, our method allows for a more comprehensive assessment of the reliability of the estimates, providing a more principled basis for informed decision-making.

% --- Algorithms ---
\begin{algorithm}[tb]
	\caption{RepFlow: Representation-Enhanced Flow Matching Training}
	\label{alg:repflow_train}
	\begin{algorithmic}[1]
		\REQUIRE Observational data $\mathcal{D}$, Representation network $\phi$ with initial parameters $\theta_\phi$, Velocity network $v$ initial parameters $\theta_v$.
		\WHILE{not converged}
		\STATE Sample mini-batch $\mathcal{B} = \{(x_i, a_i, y_i)\}_{i=1}^m$ from $\mathcal{D}$.
	
		\STATE Compute representations: $Z = \phi(X)$.
		\STATE Compute balancing loss $\mathcal{L}_{bal}$ using Eq.~\eqref{bal_loss}.
		\STATE Sample $t$, noise $y_1$, and construct interpolant $\psi_t$.
		\STATE Compute target velocity $u_t$.
		\STATE Predict velocity $\hat{v} = v(\psi_t, t; Z, A)$.
		\STATE Compute flow matching loss using Eq.~\eqref{flow_loss}.
	
		\STATE $\theta_v \leftarrow \theta_v - \eta \nabla_{\theta_v} \mathcal{L}_{flow}$
		\STATE $\mathcal{L}_{total} = \mathcal{L}_{flow} + \lambda \mathcal{L}_{bal}$
		\STATE $\theta_\phi \leftarrow \theta_\phi - \eta \nabla_{\theta_\phi} \mathcal{L}_{total}$
		\ENDWHILE
		\STATE {\bfseries Output:} Trained networks $\phi$ and $v$.
	\end{algorithmic}
\end{algorithm}

\begin{table}[ht]
	\centering
	\caption{In-sample and out-of-sample performance (mean ± standard error) of $\epsilon_{\text{PEHE}}$ and $\epsilon_{\text{ATE}}$ on the IHDP dataset. Results are averaged over 100 repetitions.}
	\label{tab:ihdp_results}
	\setlength{\tabcolsep}{3pt}  % 默认 6pt，压缩列间距
	\begin{tabular}{l@{\hspace{4pt}}cccc}
		\hline
		Method &
		$\sqrt{\epsilon_{\text{PEHE}}^{\text{in}}}$ &
		$\epsilon_{\text{ATE}}^{\text{in}}$ &
		$\sqrt{\epsilon_{\text{PEHE}}^{\text{out}}}$ &
		$\epsilon_{\text{ATE}}^{\text{out}}$ \\
		\hline
		S-learner &1.10$\pm$.06 & .21$\pm$.08 &1.17$\pm$.10 &.24$\pm$.11\\
		T-learner &0.72$\pm$.05 &.10$\pm$.06 &0.75$\pm$.07 & .10$\pm$.07 \\
		DR-learner &0.74$\pm$.17 & .17$\pm$.10 & 0.72$\pm$.18 & .20$\pm$.10\\
		CF     & 3.80$\pm$.20 & .18$\pm$.01 & 3.80$\pm$.20 & .40$\pm$.03 \\
		CEVAE  & 2.70$\pm$.10 & .34$\pm$.01 & 2.60$\pm$.10 & .46$\pm$.02 \\
		CFRW & 0.71$\pm$.02 & .25$\pm$.01 & 0.76$\pm$.02 & .27$\pm$.01\\
		GANITE &1.90$\pm$.40 & .24$\pm$.09 & 2.40$\pm$.40  &.28$\pm$.10\\
		DiffPO &0.79$\pm$.16  &.16$\pm$.07 &0.84$\pm$.18 &.19$\pm$.08\\
		POflow &0.41$\pm$.08 &\textbf{.07$\pm$.05} &0.45$\pm$.10 &\textbf{.09$\pm$.07}\\
		\textbf{Repflow (Ours)} &\textbf{0.31$\pm$.05} &\textbf{.08$\pm$.05} &\textbf{0.33$\pm$.09} &\textbf{.08$\pm$.06}\\
		\hline
	\end{tabular}
\end{table}
\section{Experiments}
%\begin{table}[ht]
%	\centering
%	\caption{In-sample and out-of-sample performance (mean ± standard error) of $\epsilon_{\text{PEHE}}$ and $\epsilon_{\text{ATE}}$ on the IHDP dataset. Results are averaged over 100 repetitions.}
%	\label{tab:ihdp_results}
%	\setlength{\tabcolsep}{3pt}  % 默认 6pt，压缩列间距
%	\begin{tabular}{l@{\hspace{4pt}}cccc}
%		\hline
%		Method &
%		$\sqrt{\epsilon_{\text{PEHE}}^{\text{in}}}$ &
%		$\epsilon_{\text{ATE}}^{\text{in}}$ &
%		$\sqrt{\epsilon_{\text{PEHE}}^{\text{out}}}$ &
%		$\epsilon_{\text{ATE}}^{\text{out}}$ \\
%		\hline
%		S-learner &1.10$\pm$.06 & .21$\pm$.08 &1.17$\pm$.10 &.24$\pm$.11\\
%		T-learner &0.72$\pm$.05 &.10$\pm$.06 &0.75$\pm$.07 & .10$\pm$.07 \\
%		DR-learner &0.74$\pm$.17 & .17$\pm$.10 & 0.72$\pm$.18 & .20$\pm$.10\\
%		CF     & 3.80$\pm$.20 & .18$\pm$.01 & 3.80$\pm$.20 & .40$\pm$.03 \\
%		CEVAE  & 2.70$\pm$.10 & .34$\pm$.01 & 2.60$\pm$.10 & .46$\pm$.02 \\
%		CFRW & 0.71$\pm$.02 & .25$\pm$.01 & 0.76$\pm$.02 & .27$\pm$.01\\
%		GANITE &1.90$\pm$.40 & .24$\pm$.09 & 2.40$\pm$.40  &.28$\pm$.10\\
%		DiffPO &0.79$\pm$.16  &.16$\pm$.07 &0.84$\pm$.18 &.19$\pm$.08\\
%		POflow &0.41$\pm$.08 &\textbf{.07$\pm$.05} &0.45$\pm$.10 &\textbf{.09$\pm$.07}\\
%		\textbf{Repflow (Ours)} &\textbf{0.31$\pm$.05} &\textbf{.08$\pm$.05} &\textbf{0.33$\pm$.09} &\textbf{.08$\pm$.06}\\
%		\hline
%	\end{tabular}
%\end{table}
In this section, we conduct extensive experiments to evaluate the performance of our proposed RepFlow model. We aim to demonstrate that by combining representation enhancement with flow matching, our method achieves superior precision in causal effect estimation across diverse benchmark datasets.

\subsection{Experimental Setup}
\begin{table*}[t] % 使用 table* 环境实现跨双栏，[t] 通常是跨栏表格的最佳位置
	\centering
	\caption{Results for benchmarking causal effect estimation on ACIC 2018 (24 datasets) for both in \& out of sample, respectively. Reported: \% of runs with the Top 1 and 3 performances.}
	\label{tab:cate_benchmark}
	\vskip 0.15in % 遵循 ICML 规范，在标题和表格间增加一点间距
	\begin{small} % 跨双栏表格建议使用 small 或常用字号，避免 resizebox 导致字号不统一
		\resizebox{\textwidth}{!}{% % 将 \linewidth 改为 \textwidth
			\begin{tabular}{lcccccccc}
				\hline
				Method & TOP1$\sqrt{\epsilon_{\text{PEHE}}^{\text{in}}}$\% & TOP3$\sqrt{\epsilon_{\text{PEHE}}^{\text{in}}}$\% & TOP1$\epsilon_{\text{ATE}}^{\text{in}}$\% & TOP3$\epsilon_{\text{ATE}}^{\text{in}}$\% & TOP1$\sqrt{\epsilon_{\text{PEHE}}^{\text{out}}}$\%
				& TOP3$\sqrt{\epsilon_{\text{PEHE}}^{\text{out}}}\%$& TOP1$\epsilon_{\text{ATE}}^{\text{out}}$\%
				& TOP3$\epsilon_{\text{ATE}}^{\text{out}}$\%
				\\
				\hline
				
				S-learner &0.00 & 4.17 &0.00  &4.17 & 0.00&4.17&0.00&4.17\\
				T-learner &0.00 &8.33 &0.00 & 8.33 &0.00 & 4.17 &0.00  &4.17 \\
				DR-learner &8.33 & 33.33 & 0.00 & 12.50 &4.17&25.00&0.00&16.67\\
				
				CF     & 0.00 & 8.33 & 12.50 &37.50& 0.00&4.17&0.00&4.17 \\
				CEVAE  &0.00 &20.83 & 0.00 & 16.67&0.00 &16.67 & 4.17&8.33 \\
				CFRW & 12.50 & 33.33& 4.17 & 37.50&25.00&29.17&4.17&37.50\\
				GANITE &8.33&58.33 & 8.33  &33.33&8.33&66.67&8.33&45.83\\
				DiffPO &20.83 &50.00&25.00 &58.33&25.00&62.50&20.83&66.67\\
				POflow &12.50 &16.67 &8.33&29.17&4.17&12.50&4.17&20.83\\
				\textbf{Repflow(Ours)} &\textbf{37.50} &\textbf{66.67} &\textbf{41.67} &\textbf{62.50}&\textbf{33.33}&\textbf{79.17}&\textbf{58.33}&\textbf{87.50}\\
				
				\hline
			\end{tabular}
		}
	\end{small}
	% \vskip -0.1in % 稍微收缩下方间距，节省空间
\end{table*}
\noindent\textbf{Datasets.} Following previous studies \cite{shalit2017estimating,ma2024diffpo,wu2025po}, we evaluate our model on two publicly available causal inference benchmarks and compare it with other generative models for distribution estimation on synthetic data. The data generation settings are detailed in Appendix~\ref{data_setting}.

\begin{itemize} \item \textbf{IHDP Dataset:} This is a semi-synthetic dataset based on  the Infant Health and Development Program, where the covariates are obtained by a randomized experiment assessing the impact of home visits by specialist doctors on the cognitive test scores for premature infants \cite{hill2011bayesian}. The dataset comprises 747 units (including 139 treated and 608 control) and 25 pre-treatment covariates measuring various aspects of the children and their mothers. It is widely used to  evaluate a model's accuracy in small sample scenario.

\item \textbf{ACIC 2018 Dataset:} The ACIC 2018 dataset is a collection of semi-synthetic datasets released for the 2018 Atlantic Causal Inference Conference Challenge. It is built upon the Linked Birth and Infant Death Data, featuring 177 dimensional medical covariates, and simulates treatment assignments and potential outcomes across 24 distinct generative mechanisms \cite{dorie2019automated}. Sample sizes range from 1,000 to 50,000, testing the scalability and robustness of estimators. \end{itemize}

\noindent\textbf{Baselines.}
We compare our method with representative approaches from four categories:
(i) model-agnostic learners, including S-learner \cite{athey2015machine}, T-learner \cite{kunzel2019metalearners}, and DR-learner \cite{kennedy2023towards};
(ii) model-specific methods, including Causal Forest (CF) \cite{wager2018estimation} and CEVAE \cite{louizos2017causal};
(iii) representation-based models, including CFRW \cite{shalit2017estimating};
and (iv) deep generative models, including GANITE \cite{yoon2018ganite}, DiffPO \cite{ma2024diffpo}, and POflow \cite{wu2025po}. Implementation details of all methods are provided in Appendix~\ref{imp}.

%\noindent\textbf{Baselines.} To comprehensively evaluate performance, our approach is compared with several state-of-the-art methods, including model-agnostic learners, model-specific approaches, representation-based methods, and deep generative models.
%
%\begin{itemize}\item S-learner \cite{athey2015machine}: S-learner treats the treatment assignment as an additional input feature within a single regression model to estimate potential outcomes.
%\item T-learner \cite{kunzel2019metalearners}:  T-learner utilizes two separate regressors for each treatment group.
%\item DR-learner \cite{kennedy2023towards}: DR-learner integrates outcome modeling and propensity scores through a two-stage process to estimate the CATE.
%\item CF \cite{wager2018estimation}: CF extends the random forest algorithm to partition the covariate space for direct treatment effect estimation.
%\item CFRW \cite{shalit2017estimating}: CFRW is a representation-based conditional regression framework for estimating CATE.
%\item CEVAE \cite{louizos2017causal}: CEVAE employs variational autoencoders to recover latent confounders and estimate potential outcomes.
%\item GANITE \cite{yoon2018ganite}: GANITE utilizes generative adversarial networks to generate potential outcomes. 
%\item DiffPO \cite{ma2024diffpo}: DiffPO leverages diffusion probabilistic models to estimate the conditional distribution of potential outcomes.
%\item POflow \cite{wu2025po}: POflow applies continuous normalizing flows to model potential outcome distributions.
%\end{itemize}

\noindent\textbf{Performance Metrics and Settings.} We evaluate causal effect estimation using three standard metrics, where lower values indicate better performance:  

\begin{itemize}
\item \textbf{Precision in the Estimation of Heterogeneous Effects ($\epsilon_{\text{PEHE}}$):} Measures the root mean squared error in estimating the CATE: $\sqrt{\epsilon_{\text{PEHE}}} = \sqrt{\frac{1}{N} \sum_{i=1}^N ( \hat{\tau}(x_i) - \tau(x_i) )^2}$.

\item \textbf{ATE Error ($\epsilon_{\text{ATE}}$):} Measures the absolute error in estimating the ATE: $\epsilon_{\text{ATE}} = |\tau - \hat{\tau}|$.
\end{itemize}

\begin{itemize}
    \item \textbf{Empirical Wasserstein Distance ($W_1$):} Measures the discrepancy between the true and estimated distributions of potential outcomes, defined as:
    \begin{equation}
        W_1(p_1, p_2) = \frac{1}{m} \sum_{i=1}^m \| \hat{y}^{(i)} - y^{(i)} \|,\nonumber
    \end{equation}
\end{itemize}
 where $p_1 = p_{\theta}(Y|X,A)$ and $p_2 = p(Y|X,A)$ denote the estimated and true conditional distributions of POs, respectively.
\begin{figure}[t]
	\centering
	% 使用 \columnwidth 确保图片不会超出单栏边界
	\includegraphics[width=\columnwidth]{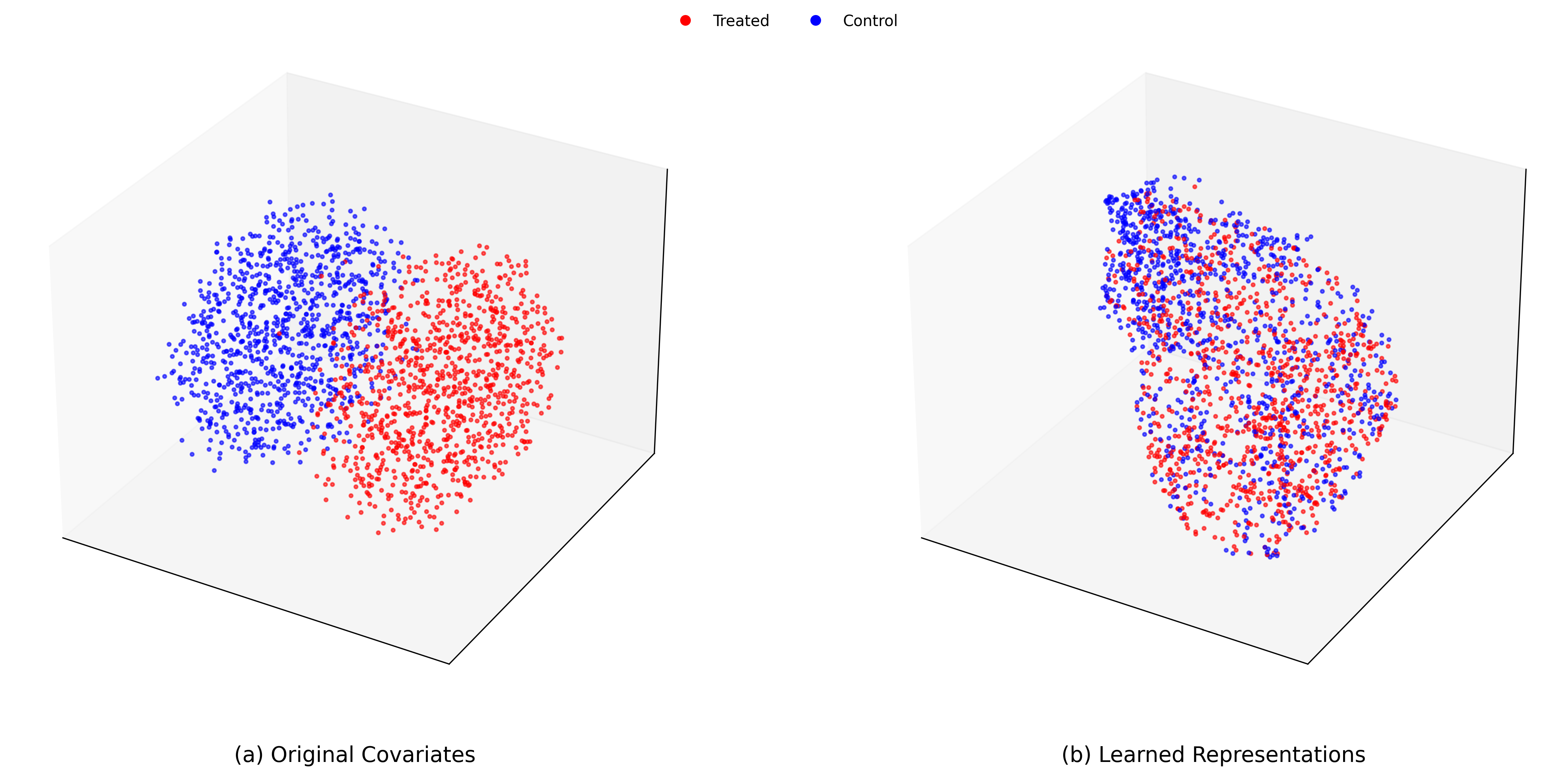} 
	\caption{Illustrations of dimensionality reduction results via UMAP are shown: the left panel displays original covariates, and the right panel shows learned representations. Red points indicate treated samples, while blue points correspond to control samples.}
	\label{fig:umap_comparison}
\end{figure}
\begin{table}[htbp]
  \centering
  \caption{Results showing in \& out-of-sample empirical Wasserstein distance (i.e., $\hat{W}^1_{\text{in}}$ and $\hat{W}^1_{\text{out}}$) for potential outcomes on the synthetic datasets.}
  \resizebox{\linewidth}{!}{% 自适应宽度（可根据需求调整）
  \begin{tabular}{l|cc|cc}
    \toprule
    & \multicolumn{2}{c|}{Setting A} & \multicolumn{2}{c}{Setting B} \\
    & $\hat{W}^1_{\text{in}}$ & $\hat{W}^1_{\text{out}}$ & $\hat{W}^1_{\text{in}}$ & $\hat{W}^1_{\text{out}}$ \\
    \midrule
  
    GANITE & $1.27 \pm 0.09$ & $1.30 \pm 0.08$ & $1.29 \pm 0.23$ & $1.32 \pm 0.22$ \\
    DiffPO & $1.15 \pm 0.15$ & $1.17 \pm 0.11$ & $1.25 \pm 0.26$ & $1.35 \pm 0.23$ \\
    POflow & $1.08 \pm 0.06$ & $1.14 \pm 0.05$ & $1.26 \pm 0.16$ & $1.28 \pm 0.15$ \\
    \midrule
    \textbf{Repflow (ours)} & $\mathbf{1.00 \pm 0.15}$ & $\mathbf{1.03 \pm 0.12}$ & $\mathbf{1.11 \pm 0.12}$ & $\mathbf{1.13 \pm 0.14}$ \\
    \bottomrule
   
  \end{tabular}
  }
  \label{tab:synthetic_wasserstein}
\end{table}

\subsection{Results and Analysis} 
Table\ref{tab:ihdp_results} displays the performance on the IHDP dataset. Traditional learners and existing generative baselines demonstrate limited precision in estimating potential outcomes, whereas RepFlow consistently achieves the lowest estimation error in both within-sample and out-of-sample settings. For the more complex ACIC 2018 datasets, Table \ref{tab:cate_benchmark} indicates that RepFlow maintains superior stability across diverse experimental scenarios and outperforms state-of-the-art methods in the majority of runs.
Table \ref{tab:synthetic_wasserstein} reports empirical Wasserstein distances on synthetic datasets. Across both Setting A and Setting B, RepFlow attains the smallest in-sample ($\hat{W}^1_{\text{in}}$) and out-of-sample ($\hat{W}^1_{\text{out}}$) distances, indicating highly accurate modeling of PO distributions.

To visualize the effect of representation network in \text{RepFlow} for mitigating selection bias, we employ Uniform Manifold Approximation and Projection (UMAP) \cite{mcinnes2018umap} to project the data into a 3D space. As illustrated in Figure~\ref{fig:umap_comparison}(a), the original covariates of the IHDP dataset show a distinct distributional shift between the treated and control groups. In contrast, Figure~\ref{fig:umap_comparison}(b) shows that our \text{RepFlow} successfully maps these features into a balanced representation space where the two distributions are significantly overlapped. 

\begin{table}[htbp]
  \centering
  \caption{Ablation Study on the IHDP Dataset.}
  \label{tab:abl_ihdp_results}
  % 1. 降低列与列之间的间距（默认通常是 6pt，这里设为 2pt）
  \setlength{\tabcolsep}{2.5pt} 
  % 2. 稍微缩小字号
  \begin{small} 
  % 3. 最后的保底手段：强制缩放到单栏宽度
  \resizebox{\columnwidth}{!}{%
  \begin{tabular}{lcccc}
    \hline
    Method & $\sqrt{\epsilon_{\text{PEHE}}^{\text{in}}}$ & $\epsilon_{\text{ATE}}^{\text{in}}$ & $\sqrt{\epsilon_{\text{PEHE}}^{\text{out}}}$ & $\epsilon_{\text{ATE}}^{\text{out}}$\\
    \hline
    Repflow w/o rep & 0.47$\pm$.06 & .09$\pm$.08 & 0.50$\pm$.08 & .11$\pm$.08 \\
    Repflow with $\lambda=0$ & 0.40$\pm$.06 & .09$\pm$.07 & 0.41$\pm$.10 & .10$\pm$.07 \\
    Repflow with mmd &0.34$\pm$.05 &.08$\pm$.06 &0.36$\pm$.10 &.09$\pm$.07\\
    Repflow & 0.31$\pm$.05 & .08$\pm$.05 & 0.33$\pm$.09 & .08$\pm$.06 \\
    \hline
  \end{tabular}
  }
  \end{small}
\end{table}

\subsection{Ablation Study}

To assess the contribution of each component in RepFlow, we conduct ablation experiments on the IHDP dataset by considering the following variants:
(i) RepFlow w/o rep, which removes the representation learning module;
(ii) RepFlow with $\lambda=0$, which disables the distribution regularization in flow matching;
and (iii) RepFlow with MMD, which replaces the wasserstein distance with an MMD-based distribution alignment.
The full model is denoted as RepFlow.

The results are summarized in Table~\ref{tab:abl_ihdp_results}.
We observe that removing the representation module leads to a clear degradation in both in-sample and out-of-sample performance, highlighting the importance of learning balanced and informative representations.
Disabling the distribution regularization by setting $\lambda=0$ also results in inferior performance, especially in out-of-sample PEHE, indicating the necessity of enforcing distributional consistency.

Overall, the full RepFlow achieves the best performance, confirming that all components contribute to accurate and robust causal effect estimation.

\subsection{Sensitivity Analysis}
\begin{figure}[t]
	\centering
	\begin{subfigure}{0.48\columnwidth}
		\centering
		\includegraphics[width=\linewidth]{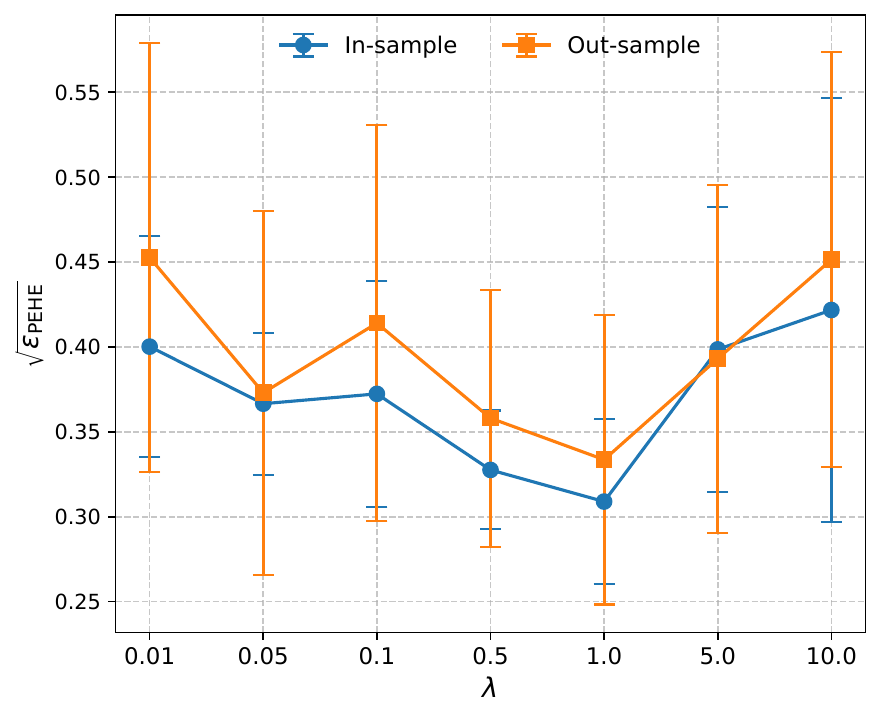}
		%\caption{PEHE vs. $\lambda$}
	\end{subfigure}
	\hfill
	\begin{subfigure}{0.48\columnwidth}
		\centering
		\includegraphics[width=\linewidth]{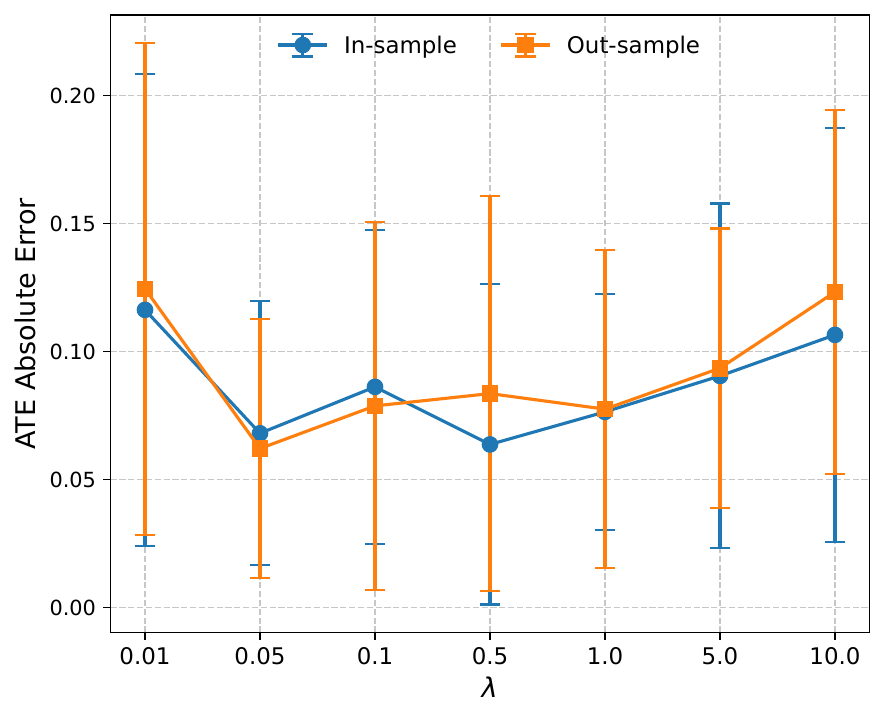}
		%\caption{ATE error vs. $\lambda$}
	\end{subfigure}
	
	\vspace{0.5em}
	
	\begin{subfigure}{0.48\columnwidth}
		\centering
		\includegraphics[width=\linewidth]{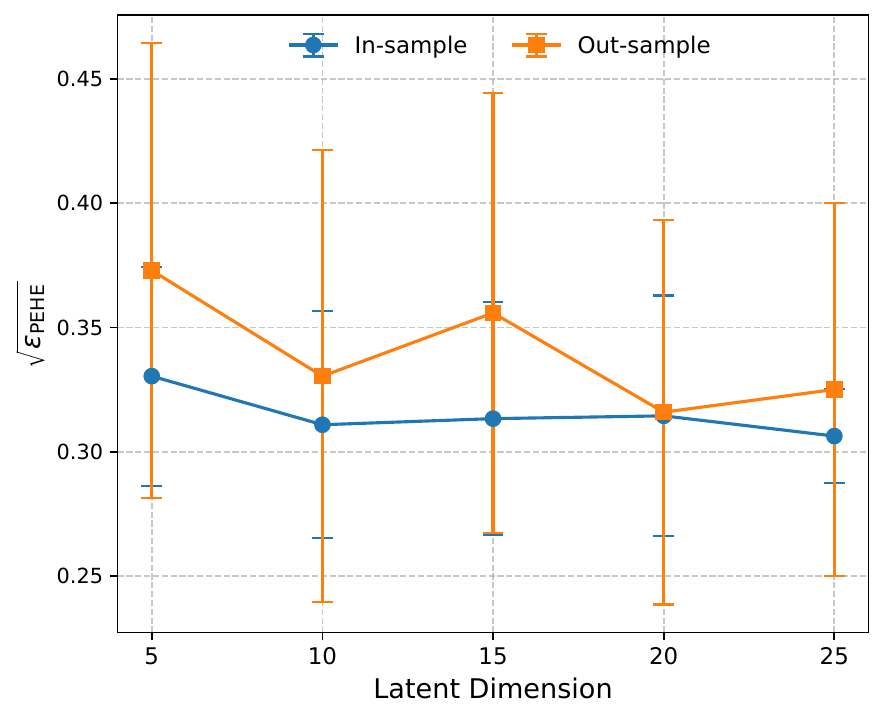}
		%\caption{PEHE vs. latent dim}
	\end{subfigure}
	\hfill
	\begin{subfigure}{0.48\columnwidth}
		\centering
		\includegraphics[width=\linewidth]{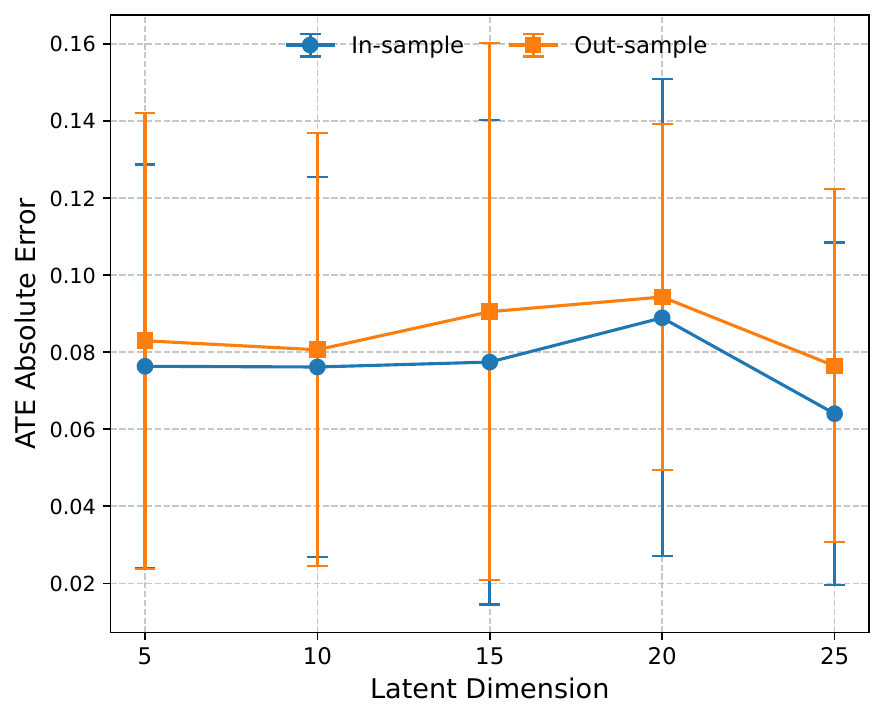}
		%\caption{ATE error vs. latent dim}
	\end{subfigure}
	\caption{Results of varying values of parameters on IHDP dataset.}
	\label{fig:parameters}
\end{figure}
We explore the sensitivity of the hyper-parameters on the IHDP dataset and the results are shown in Figure \ref{fig:parameters}.

When $\lambda$ is too small, insufficient regularization leads to poor generalization, while overly large $\lambda$ restricts model flexibility and degrades performance. Moderate values of $\lambda$  yield the best in-sample and out-of-sample performance, balancing fitting accuracy and distributional consistency.

Performance remains stable across a wide range of latent dimensions, with only minor fluctuations in PEHE and ATE error, indicating RepFlow is not overly sensitive to this hyperparameter.

Overall, RepFlow demonstrates robustness to hyperparameter variations, with stable performance across settings and optimal results under moderate distribution regularization.

\section{Conclusion}

We presented RepFlow, a framework that integrates representation learning with conditional flow matching to estimate potential outcome distributions. By minimizing the latent Wasserstein distance and enforcing $L_{2}$ normalization, RepFlow effectively mitigates selection bias while ensuring numerical stability. Furthermore, we provide a rigorous distributional generalization bound that establishes a principled theoretical foundation for our joint optimization strategy. Comprehensive evaluations across synthetic, IHDP, and ACIC 2018 datasets demonstrate that RepFlow consistently outperforms state-of-the-art baselines.

A limitation of our current framework is the reliance on the unconfoundedness assumption, which presumes that treatment assignment is conditionally independent of potential outcomes given the observed covariates.
Our future work focuses on extending this framework to more general settings with unmeasured confounding, where causal effect identifiability requires the introduction of additional structural assumptions. We explore extensions based on instrumental variable (IV) \cite{hartford2017deep} and proximal causal inference methods \cite{cui2024semiparametric} to support causal estimation under unmeasured confounding.

\section*{Impact Statement}

This research bridges representation learning and generative flows to enhance the reliability of causal effect estimation from observational data. By effectively mitigating selection bias and capturing  potential outcome distributions, our method could be applied to a wide range of applications, such as decision-making in healthcare, public policy, business.

\nocite{langley00}

\bibliography{example_paper}
\bibliographystyle{icml2026}

%%%%%%%%%%%%%%%%%%%%%%%%%%%%%%%%%%%%%%%%%%%%%%%%%%%%%%%%%%%%%%%%%%%%%%%%%%%%%%%
%%%%%%%%%%%%%%%%%%%%%%%%%%%%%%%%%%%%%%%%%%%%%%%%%%%%%%%%%%%%%%%%%%%%%%%%%%%%%%%
% APPENDIX
%%%%%%%%%%%%%%%%%%%%%%%%%%%%%%%%%%%%%%%%%%%%%%%%%%%%%%%%%%%%%%%%%%%%%%%%%%%%%%%
%%%%%%%%%%%%%%%%%%%%%%%%%%%%%%%%%%%%%%%%%%%%%%%%%%%%%%%%%%%%%%%%%%%%%%%%%%%%%%%
\newpage
\appendix
\onecolumn

\section{Proof}\label{proof}
\begin{definition}
	Let $\Phi: \mathcal{X} \rightarrow \mathcal{Z}$ be the representation network, and let $\hat{p}_{\theta}(Y|\phi(x), a)$ be the conditional distribution of potential outcomes generated by the Repflow model. For a unit $x \in \mathcal{X}$ and treatment assignment $a \in \{0, 1\}$, the  distributional distance is defined as:
	\begin{equation}
		d_{\theta, \phi}(x, a) = \mathcal{D} \left( p(Y|x,a), \hat{p}_{\theta}(Y|\phi(x), a) \right), \nonumber
	\end{equation}
	where $\mathcal{D}(\cdot, \cdot)$ is a metric on the space of probability distributions, such as the Wasserstein distance $\mathcal{W}_1$.
\end{definition}

\begin{assumption}
	The representation function $\Phi: \mathcal{X} \rightarrow \mathcal{Z}$ is one-to-one. We define $\Psi: \mathcal{Z} \rightarrow \mathcal{X}$ to be the inverse of $\Phi$, such that $\Psi(\Phi(x)) = x$ for all $x \in \mathcal{X}$ . This assumption ensures that the transformation from the covariate space to the latent space preserves all necessary information for causal identification and allows for the change of variables in the error bound derivation.
\end{assumption}

% --- Definition 2: Expected Distributional Losses ---
\begin{definition}
	The expected factual and counterfactual distributional losses for the model $(\theta, \phi)$ are defined respectively as:
	\begin{equation}
		\epsilon_{F}(\theta, \phi) = \int_{\mathcal{X} \times \{0,1\}} d_{\theta, \phi}(x, a) p(x, a) dx da
	\end{equation}
	\begin{equation}
		\epsilon_{CF}(\theta, \phi) = \int_{\mathcal{X} \times \{0,1\}} d_{\theta, \phi}(x, a) p(x, 1-a) dx da
	\end{equation}
\end{definition}

% --- Definition 3: Group-specific Distributional Losses ---
\begin{definition} \label{df3}
	Conditioned on the treatment assignment, the factual and counterfactual losses for each group are:
	\begin{align}
		\text{Factual: } & \epsilon_{F}^{a=1} = \int_{\mathcal{X}} d_{\theta, \phi}(x, 1) p(x|a=1) dx, \quad \epsilon_{F}^{a=0} = \int_{\mathcal{X}} d_{\theta, \phi}(x, 0) p(x|a=0) dx \\
		\text{Counterfactual: } & \epsilon_{CF}^{a=1} = \int_{\mathcal{X}} d_{\theta, \phi}(x, 1) p(x|a=0) dx, \quad \epsilon_{CF}^{a=0} = \int_{\mathcal{X}} d_{\theta, \phi}(x, 0) p(x|a=1) dx
	\end{align}
	where $p^{a=1}(x)$ and $p^{a=0}(x)$ denote the treated and control distributions respectively.
\end{definition}

% --- Definition 4: Distributional PEHE ---
\begin{definition}
	The expected estimation of distributional is defined as the total error in estimating the potential outcome distributions across the entire population:
	\begin{equation}
		\epsilon_{D}(\theta, \phi) = \int_{\mathcal{X}} \left( d_{\theta, \phi}(x, 1) + d_{\theta, \phi}(x, 0) \right) p(x) dx
	\end{equation}
\end{definition}

\begin{definition}\label{dfipm}
	Let $G$ be a family of real-valued functions $g: \mathcal{Z} \rightarrow \mathbb{R}$ defined over the representation space $\mathcal{Z}$. For a pair of probability distributions $p, q$ over $\mathcal{Z}$, the Integral Probability Metric (IPM) is defined as:
	\begin{equation}
		\mathrm{IPM}_G(p, q) = \sup_{g \in G} \left| \int_{\mathcal{Z}} g(z) p(z) dz - \int_{\mathcal{Z}} g(z) q(z) dz \right|
	\end{equation}
	The $\mathrm{IPM}_G(\cdot, \cdot)$ defines a pseudo-metric on the space of probability functions over $\mathcal{Z}$. For sufficiently large function families, it becomes a proper metric. Specifically:
	\begin{itemize}
		\item If $G$ is the set of 1-Lipschitz functions, $\mathrm{IPM}_G$ corresponds to the \textbf{Wasserstein distance $W_{1}$}, used in RepFlow to mitigate selection bias.
		\item If $G$ is the set of unit norm functions in a universal Reproducing Kernel Hilbert Space (RKHS), $\mathrm{IPM}_G$ corresponds to the \textbf{Maximum Mean Discrepancy (MMD)}.
	\end{itemize}

\end{definition}

\begin{lemma}\label{decom}
	Let $u := p(a=1)$ be the proportion of treated units in the population. The expected factual and counterfactual distributional losses satisfy:
	\begin{align}
		\epsilon_{F}(\theta, \phi) &= u \cdot \epsilon_{F}^{a=1}(\theta, \phi) + (1-u) \cdot \epsilon_{F}^{a=0}(\theta, \phi) \\
		\epsilon_{CF}(\theta, \phi) &= (1-u) \cdot \epsilon_{CF}^{a=1}(\theta, \phi) + u \cdot \epsilon_{CF}^{a=0}(\theta, \phi)
	\end{align}
\end{lemma}

\begin{proof}
	By the law of total expectation, we can decompose the joint distribution $p(x, a)$ as $p(a)p(x|a)$. Noting that $p(a=1) = u$ and $p(a=0) = 1-u$, the factual loss becomes:
	\begin{align*}
		\epsilon_{F} &= \int_{\mathcal{X} \times \{0,1\}} d_{\theta, \phi}(x, a) p(x, a) dx da \\
		&= p(a=1) \int_{\mathcal{X}} d_{\theta, \phi}(x, 1) p(x|a=1) dx + p(a=0) \int_{\mathcal{X}} d_{\theta, \phi}(x, 0) p(x|a=0) dx \\
		&= u \cdot \epsilon_{F}^{a=1} + (1-u) \cdot \epsilon_{F}^{a=0}
	\end{align*}
	Similarly, for the counterfactual loss, we evaluate the distance under the flipped treatment assignment $1-a$:
	\begin{align*}
		\epsilon_{CF} &= \int_{\mathcal{X} \times \{0,1\}} d_{\theta, \phi}(x, a) p(x, 1-a) dx da \\
		&= p(a=0) \int_{\mathcal{X}} d_{\theta, \phi}(x, 1) p(x|a=0) dx + p(a=1) \int_{\mathcal{X}} d_{\theta, \phi}(x, 0) p(x|a=1) dx \\
		&= (1-u) \cdot \epsilon_{CF}^{a=1} + u \cdot \epsilon_{CF}^{a=0}
	\end{align*}
	which completes the proof.
\end{proof}

\begin{lemma}\label{cfdb} 
	Let $\Phi: \mathcal{X} \rightarrow \mathcal{Z}$ be an invertible representation with $\Psi$ as its inverse. Let $p_{\Phi}^{a=1}, p_{\Phi}^{a=0}$ be the distributions induced by $\Phi$ over $\mathcal{Z}$ . Let $u = p(a=1)$ and $G$ be a family of functions $g: \mathcal{Z} \rightarrow \mathbb{R}$. Denote by $\mathrm{IPM}_G(\cdot, \cdot)$ the integral probability metric induced by $G$. 
	
	Assume there exists a constant $B_{\Phi} > 0$ such that for $a \in \{0, 1\}$, the normalized distribution distance functions $g_{\theta, \phi}(z, a) := \frac{1}{B_{\Phi}} d_{\theta, \phi}(\Psi(z), a)$ belong to $G$. Then we have:
	\begin{equation}
		\epsilon_{CF}(\theta, \phi) \leq (1-u)\epsilon_{F}^{a=1}(\theta, \phi) + u\epsilon_{F}^{a=0}(\theta, \phi) + B_{\Phi} \cdot \mathrm{IPM}_G(p_{\Phi}^{a=1}, p_{\Phi}^{a=0})
	\end{equation}
\end{lemma}

\begin{proof}
	Consider the discrepancy between the total counterfactual distributional loss and a weighted combination of group-specific factual distributional losses. Let:
	\begin{equation}
		\Delta = \epsilon_{CF}(\theta, \phi) - \left[ (1-u)\epsilon_{F}^{a=1}(\theta, \phi) + u\epsilon_{F}^{a=0}(\theta, \phi) \right] \notag
	\end{equation}
	By the decomposition properties of counterfactual and factual losses defined in Lemma \ref{decom}, we have:
	\begin{align}
		\Delta &= \left[ (1-u)\epsilon_{CF}^{a=1} + u\epsilon_{CF}^{a=0} \right] - \left[ (1-u)\epsilon_{F}^{a=1} + u\epsilon_{F}^{a=0} \right] = (1-u) \left[ \epsilon_{CF}^{a=1} - \epsilon_{F}^{a=1} \right] + u \left[ \epsilon_{CF}^{a=0} - \epsilon_{F}^{a=0} \right] \label{eq:repflow_diff_a}
	\end{align}
	Using the integral definitions \ref{df3} of factual and counterfactual losses, we express Eq. \ref{eq:repflow_diff_a} as integrals over the covariate space $\mathcal{X}$. Performing a change of variables to the representation space $\mathcal{Z}$ via the invertible mapping $\Phi$ and its inverse $\Psi$:
	\begin{align}
		\Delta &= (1-u) \int_{\mathcal{Z}} d_{\theta, \phi}(\Psi(z), 1) \left( p_{\Phi}^{a=0}(z) - p_{\Phi}^{a=1}(z) \right) dz + u \int_{\mathcal{Z}} d_{\theta, \phi}(\Psi(z), 0) \left( p_{\Phi}^{a=1}(z) - p_{\Phi}^{a=0}(z) \right) dz \label{eq:z_integral_a}
	\end{align}
	Based on the premise that there exists a constant $B_{\Phi} > 0$ such that the normalized distributional distance function $g_{\theta, \phi}(z, a) := \frac{1}{B_{\Phi}} d_{\theta, \phi}(\Psi(z), a)$ belongs to the function family $G$, the integrals are rewritten as:
	\begin{align}
		\Delta &= B_{\Phi} (1-u) \int_{\mathcal{Z}} \frac{1}{B_{\Phi}} d_{\theta, \phi}(\Psi(z), 1) \left( p_{\Phi}^{a=0}(z) - p_{\Phi}^{a=1}(z) \right) dz  + B_{\Phi} u \int_{\mathcal{Z}} \frac{1}{B_{\Phi}} d_{\theta, \phi}(\Psi(z), 0) \left( p_{\Phi}^{a=1}(z) - p_{\Phi}^{a=0}(z) \right) dz
	\end{align}
	By applying the definition \ref{dfipm} of the Integral Probability Metric (IPM) for family $G$ , where $|\int_{\mathcal{Z}} g(z) ( p_{\Phi}^{a=1} - p_{\Phi}^{a=0} ) dz| \leq \mathrm{IPM}_G(p_{\Phi}^{a=1}, p_{\Phi}^{a=0})$, and noting that $(1-u) + u = 1$:
	\begin{align}
		\Delta 
		&\leq B_{\Phi} (1-u) \left| \int_{\mathcal{Z}} g(z, 1) \bigl( p_{\Phi}^{a=0}(z) - p_{\Phi}^{a=1}(z) \bigr) \, dz \right| 
		+ B_{\Phi} u \left| \int_{\mathcal{Z}} g(z, 0) \bigl( p_{\Phi}^{a=1}(z) - p_{\Phi}^{a=0}(z) \bigr) \, dz \right| \nonumber \\
		&= B_{\Phi} (1-u) \left| \int_{\mathcal{Z}} g(z, 1) \bigl( p_{\Phi}^{a=1}(z) - p_{\Phi}^{a=0}(z) \bigr) \, dz \right| 
		+ B_{\Phi} u \left| \int_{\mathcal{Z}} g(z, 0) \bigl( p_{\Phi}^{a=1}(z) - p_{\Phi}^{a=0}(z) \bigr) \, dz \right| \nonumber \\
		&\leq B_{\Phi} (1-u) \cdot \mathrm{IPM}_{G}\bigl(p_{\Phi}^{a=1}, p_{\Phi}^{a=0}\bigr) 
		+ B_{\Phi} u \cdot \mathrm{IPM}_{G}\bigl(p_{\Phi}^{a=1}, p_{\Phi}^{a=0}\bigr) \nonumber \\
		&= B_{\Phi} \bigl[ (1-u) + u \bigr] \cdot \mathrm{IPM}_{G}\bigl(p_{\Phi}^{a=1}, p_{\Phi}^{a=0}\bigr) \nonumber \\
		&= B_{\Phi} \cdot \mathrm{IPM}_{G}\bigl(p_{\Phi}^{a=1}, p_{\Phi}^{a=0}\bigr).
	\end{align}
	Finally, rearranging the terms yields the distributional generalization bound:
	\begin{equation}
		\epsilon_{CF} \leq (1-u)\epsilon_{F}^{a=1} + u\epsilon_{F}^{a=0} + B_{\Phi} \cdot \mathrm{IPM}_{G}(p_{\Phi}^{a=1}, p_{\Phi}^{a=0}) \notag
	\end{equation}
	This confirms that the counterfactual distributional error is bounded by the factual errors and the distribution discrepancy in the latent space.
\end{proof}

\begin{lemma}\label{trans}
	Consider two distribution functions $P_1(x)$ and $P_2(x)$ supported over $\mathcal{X}$; let $G$ be the family of 1-Lipschitz functions, $\mathcal{W}$ be the Wasserstein distance, Villani \cite{villani2008optimal} demonstrate
\begin{equation}
	\mathrm{IPM}_{G}(P_1, P_2) = \mathcal{W}(P_1, P_2).
\end{equation}
\end{lemma}
\begin{theorem}
	Let $\Phi: \mathcal{X} \rightarrow \mathcal{Z}$ be the representation network. Then, the expected distributional error over the population satisfies:
	\begin{equation}
		\epsilon_{D}(\theta, \phi) \leq \epsilon_{F}^{a=0} + \epsilon_{F}^{a=1} + B_{\Phi} \cdot \mathcal{W}\!\left(p_{\Phi}^{a=1}, p_{\Phi}^{a=0}\right),
	\end{equation}
	where $\mathcal{W}(\cdot,\cdot)$ denotes the Wasserstein distance between distributions.
\end{theorem}

\begin{proof}
	First, we expand the definition of $\epsilon_{D}$ using the law of total probability $p(x) = p(x, a=0) + p(x, a=1)$:
	\begin{align}
		\epsilon_{D} &= \int_{\mathcal{X}} \left( d_{\theta, \phi}(x, 1) + d_{\theta, \phi}(x, 0) \right) p(x) dx \notag \\
		&= \underbrace{\int_{\mathcal{X}} d_{\theta, \phi}(x, 1) p(x) dx}_{\text{I}} + \underbrace{\int_{\mathcal{X}} d_{\theta, \phi}(x, 0) p(x) dx}_{\text{II}}.
	\end{align}
	
	For term \text{I}, we decompose the population distribution $p(x)$ into factual ($a=1$) and counterfactual ($a=0$) components:
	\begin{align}
		\text{I} &= \int_{\mathcal{X}} d_{\theta, \phi}(x, 1) p(x, a=1) dx + \int_{\mathcal{X}} d_{\theta, \phi}(x, 1) p(x, a=0) dx \notag \\
		&= u \cdot \epsilon_{F}^{a=1} + (1-u) \cdot \epsilon_{CF}^{a=1}.
	\end{align}
%	Applying the Lemma \ref{cfdb} to the counterfactual part $(1-u) \cdot \epsilon_{CF}^{a=1}$:
%	\begin{equation}
%		(1-u) \epsilon_{CF}^{a=1} \leq (1-u) \left[ \epsilon_{F}^{a=1} + B_{\Phi} \mathrm{IPM}_G(p_{\Phi}^{t=1}, p_{\Phi}^{t=0}) \right]. \nonumber
%	\end{equation}
%	Substituting this back into Term I:
%	\begin{equation}
%		\text{I} \leq \epsilon_{F}^{a=1} + (1-u) B_{\Phi} \mathrm{IPM}_G(p_{\Phi}^{t=1}, p_{\Phi}^{t=0})
%	\end{equation}
	
	Similarly, for term \text{II}, we decompose $p(x)$ relative to $a=0$:
	\begin{align}
		\text{II} &= \int_{\mathcal{X}} d_{\theta, \phi}(x, 0) p(x, a=0) dx + \int_{\mathcal{X}} d_{\theta, \phi}(x, 0) p(x, a=1) dx \notag \\
		&= (1-u) \cdot \epsilon_{F}^{a=0} + u \cdot \epsilon_{CF}^{a=0}.
	\end{align}
%	Applying the Lemma\ref{cfdb} to the counterfactual part $u \cdot \epsilon_{CF}^{a=0}$:
%	\begin{equation}
%		\text{II} \leq \epsilon_{F}^{a=0} + u B_{\Phi} \mathrm{IPM}_G(p_{\Phi}^{t=1}, p_{\Phi}^{t=0})
%	\end{equation}
	
	Summing Term I and Term II, and applying the Lemmas  \ref{cfdb} and \ref{trans}:
	\begin{equation}
		\begin{aligned}
		\epsilon_{D} &= u \cdot \epsilon_{F}^{a=1} + (1-u) \cdot \epsilon_{CF}^{a=1} + (1-u) \cdot \epsilon_{F}^{a=0} + u \cdot \epsilon_{CF}^{a=0}\\
		&\leq u \cdot \epsilon_{F}^{a=1}+(1-u) \cdot \epsilon_{F}^{a=0}+\epsilon_{CF} \\
		&\leq \epsilon_{F}^{a=0} + \epsilon_{F}^{a=1} + B_{\Phi} \cdot \mathcal{W}(p_{\Phi}^{a=1}, p_{\Phi}^{a=0})
	\end{aligned}
	\end{equation}
	This concludes the proof.
\end{proof}

\section{Data Generation Settings}\label{data_setting}

To further evaluate the performance of our method in distribution estimation, we conduct the following synthetic data generation settings with high-dimensional covariates $X \in \mathbb{R}^d$ ($d=100$). The observed outcome is defined as $Y = AY^{(1)} + (1 - A)Y^{(0)}$. Both settings share a common structure for potential outcomes:
\begin{equation}
	\begin{cases} 
		Y^{(0)} = f(X) + \epsilon_0,\\
		Y^{(1)} = f(X) + \tau(X) + \epsilon_1, \nonumber
	\end{cases}
\end{equation}
where $\epsilon_0, \epsilon_1 \sim \mathcal{N}(0, 1)$. The baseline function $f(X)$ and treatment effect $\tau(X)$ are defined as:
\begin{align}
	f(X) &= 0.5 X_1^2 + 0.5 \exp\left(\frac{X_2 + X_3}{4}\right) + \sin(X_4 + X_5), \nonumber \\
	\tau(X) &= 1.0 + 0.5 X_1 - 0.5 X_2^2 + \sin(X_3). \nonumber
\end{align}

\subsection{Setting A: Propensity-based Selection}
In this setting, $X \sim \mathcal{N}(\mathbf{0}, \mathbf{I}_d)$. Selection bias is introduced via a logistic propensity score $\pi(X) = P(A=1 \mid X)$:
\begin{equation}
	\text{logit}[\pi(X)] = 0.5 + 0.5 X_1 - X_2^2 + \sin(X_3), \nonumber
\end{equation}
where $A \sim \text{Bernoulli}[\pi(X)]$.

\subsection{Setting B: Covariate Shift}
In this setting, the treatment assignment is balanced ($P(A=1) = 0.5$), but the covariate distributions for the two groups are forced to differ to simulate strong distribution shift:
\begin{equation}
	X \mid A=a \sim 
	\begin{cases} 
		\mathcal{N}(\mathbf{0}, \mathbf{I}_d), & a=0 \\
		\mathcal{N}(s \cdot \mathbf{1}_d, \mathbf{I}_d), & a=1, \nonumber
	\end{cases}
\end{equation}
where the shift parameter is set to $s = 0.5$. This setting tests the model's ability to generalize across treatment groups with limited overlap.

\section{Implementation details}\label{imp}
We implement RepFlow in PyTorch. Experiments were carried out on 4 × NVIDIA A100.The representation network comprises an initial projection layer, two residual blocks with ReLU activation, and a final linear layer that outputs $\ell_2$-normalized latent representations. The velocity network comprises an outcome embedding module, a time embedding layer, and a FiLM layer for feature conditioning. This is followed by two gated residual blocks with SiLU activation that adaptively incorporate conditioning variables, a skip-connection aggregation mechanism that sums intermediate features across all blocks to form the final hidden state, and a dual-head projection layer that predicts treatment-specific velocity fields $v_0$ and $v_1$. For model training, we adopt the Adam optimizer \cite{KingmaB14}. To ensure an equitable comparison, all methods use the same 70/20/10 train/validation/test split of the dataset, with the validation set used to select the optimal hyperparameters. 

%%%%%%%%%%%%%%%%%%%%%%%%%%%%%%%%%%%%%%%%%%%%%%%%%%%%%%%%%%%%%%%%%%%%%%%%%%%%%%%
%%%%%%%%%%%%%%%%%%%%%%%%%%%%%%%%%%%%%%%%%%%%%%%%%%%%%%%%%%%%%%%%%%%%%%%%%%%%%%%

\end{document}